\pdfoutput=1
\documentclass[lettersize,journal]{IEEEtran}
\usepackage{amsmath,amsfonts}
\usepackage{algorithmic}
\usepackage{algorithm}
\usepackage{array}
\usepackage[caption=false,font=normalsize,labelfont=sf,textfont=sf]{subfig}
\usepackage{textcomp}
\usepackage{stfloats}
\usepackage{url}
\usepackage{verbatim}
\usepackage{graphicx}
\usepackage{cite}
\usepackage{amsmath,amssymb,amsfonts}
\usepackage{graphicx}
\usepackage{booktabs}
\usepackage{multirow}
\usepackage{algorithm}
\usepackage{algorithmic}
\usepackage{amsthm}
\usepackage{bm}

\newtheorem{theorem}{Theorem}

\begin{document}

\title{Robust CurveMoE: Multi-Norm Adversarial Defense for Mixture-of-Experts Models via Mode Connectivity}

% \author{Anonymous Authors
        % <-this % stops a space
%\thanks{This paper was produced by the IEEE Publication Technology Group. They are in Piscataway, NJ.}% <-this % stops a space
%\thanks{Manuscript received April 19, 2021; revised August 16, 2021.}
% }

\author{
Xu Zhang and Ren Wang
\thanks{The authors are with the Department of Electrical and Computer Engineering, Illinois Institute of Technology, Chicago, IL, USA.}
}

% The paper headers
%\markboth{Journal of \LaTeX\ Class Files,~Vol.~14, No.~8, August~2021}%
%{Shell \MakeLowercase{\textit{et al.}}: A Sample Article Using IEEEtran.cls for IEEE Journals}

%\IEEEpubid{0000--0000/00\$00.00~\copyright~2021 IEEE}
% Remember, if you use this you must call \IEEEpubidadjcol in the second
% column for its text to clear the IEEEpubid mark.

\maketitle

\begin{abstract}
Multi-norm adversarial defense aims to protect neural networks against perturbations defined by different norm constraints, but existing methods typically optimize competing robustness objectives within a single parameter configuration, leading to substantial training cost and unfavorable robustness trade-offs. We propose Robust CurveMoE, an efficient mixture-of-experts framework that connects models specialized for different perturbation norms through a low-loss path and exploits the complementary robustness profiles of models along this path. Robust CurveMoE derives clean and norm-specialized experts from robustness-constrained curve locations and selectively expertizes only influential layers, while sharing the remaining parameters across routing paths. To further reduce curve-construction cost, we introduce contribution-guided partial updating, which selects influential curve parameters using initialization-based gradient scores. We also theoretically bound the objective gap between partial and full curve optimization. Experiments on CIFAR-100 and ImageNet-100 with WideResNet and Vision Transformer architectures show that Robust CurveMoE consistently improves clean, norm-specific, and Union accuracy over MSD and ERMC. In particular, it improves Union accuracy by 2.37 and 2.13 percentage points over the strongest baseline on CIFAR-100 and ImageNet-100, respectively. Extensive ablations further validate the effectiveness of partial updating, selective expertization, and robustness-constrained expert selection.
\end{abstract}


\section{Introduction}
\label{sec:introduction}

Deep neural networks have achieved remarkable performance in visual recognition, yet their predictions remain vulnerable to carefully crafted adversarial perturbations~\cite{szegedy2013intriguing,goodfellow2014explaining}. Adversarial training is one of the most effective defenses and is commonly formulated as a min--max optimization problem against worst-case perturbations within a prescribed norm-bounded region~\cite{madry2017towards}. However, robustness obtained under one threat model does not necessarily transfer to others. A model trained against $\ell_\infty$ perturbations, for example, may remain vulnerable to attacks constrained by $\ell_1$ or $\ell_2$ norms. This mismatch motivates multi-norm adversarial defense, which seeks to protect a model against the union of multiple perturbation geometries encountered at deployment.

Existing multi-norm defenses typically optimize multiple adversarial objectives within a single parameter configuration, for example by averaging norm-specific losses, selecting the strongest attack, or constructing updates across different perturbation sets~\cite{tramer2019adversarial,maini2020adversarial,croce2022adversarial}. Although effective, this paradigm faces an inherent difficulty: different perturbation norms induce distinct robustness objectives that may compete during optimization. Consequently, improving robustness against one norm can compromise robustness against another or reduce clean accuracy. Moreover, multi-norm adversarial optimization requires generating attacks under several threat models and therefore introduces substantial training overhead. These limitations raise a natural question: rather than forcing all robustness behaviors into a single parameter configuration, can complementary norm-specific behaviors be preserved and adaptively exploited for each input?

Mixture-of-experts (MoE) architectures provide a promising mechanism for this purpose. By maintaining multiple specialized experts and activating only a sparse subset for each input, MoEs can increase model capacity while preserving efficient conditional computation~\cite{shazeer2017outrageously,fedus2022switch,riquelme2021scaling}. Recent studies have further demonstrated the potential of MoE architectures for adversarial robustness~\cite{puigcerver2022adversarial,zhang2023robust,zhang2025optimizing}. Nevertheless, directly applying MoE to multi-norm defense remains challenging. Independently training complete experts for different perturbation norms requires repeated adversarial optimization and substantial parameter storage. More importantly, norm-specialized experts can exhibit highly imbalanced cross-norm robustness. An adversary that manipulates the routing decision toward an expert vulnerable to the current perturbation may therefore compromise the entire MoE even if that expert performs well under its specialized threat model. Effective multi-norm MoE defense thus requires both an efficient source of complementary experts and a principled mechanism for controlling their cross-norm vulnerability.

Robust mode connectivity provides such a structured source of experts. Mode-connectivity studies have shown that independently trained solutions can often be connected by nonlinear low-loss paths in parameter space~\cite{garipov2018loss,draxler2018essentially}. In adversarial settings, connectivity curves between norm-specialized models can contain intermediate models with distinct robustness characteristics~\cite{wang2023robust}. Our experiments further reveal that different regions of such curves naturally favor different objectives: models near the endpoints retain stronger robustness to their corresponding norms, while intermediate models can exhibit complementary clean and $\ell_2$ performance. Existing robust mode-connectivity methods, however, typically deploy only a single model selected from the learned curve, leaving much of this specialization unused. Using multiple complete curve models as experts would preserve this diversity but substantially increase parameter cost. In addition, constructing a robust curve itself remains expensive because multi-norm adversarial optimization is repeatedly performed at sampled curve locations.

Motivated by these observations, we propose \emph{Robust CurveMoE}, an efficient MoE framework that converts the complementary models encoded by a robust connectivity curve into input-adaptive multi-norm defense. We first connect $\ell_1$- and $\ell_\infty$-robust endpoints and optimize the resulting curve under a multi-norm adversarial objective, yielding a continuous pool of models with complementary robustness profiles. Rather than replicating complete curve models, we measure the adversarial sensitivity of candidate layers and instantiate expert-specific parameters only in the most influential layers, while sharing the remaining parameters. We further constrain expert selection to curve locations with sufficiently strong worst-case cross-norm robustness, preventing highly specialized but vulnerable models from entering the expert pool. Within this candidate set, clean and norm-specific specialists are selected to construct the experts, while a parameter-compatible model provides the shared backbone. Independent top-1 routers at the selected MoE layers then perform input-dependent sparse expert selection.

We additionally reduce the computational cost of constructing the robust connectivity curve through contribution-guided partial updating. Instead of optimizing all trainable curve parameters, Robust CurveMoE estimates their contribution to the multi-norm objective at initialization and updates only the most influential subset. We theoretically show that the objective gap between full and restricted curve optimization is bounded by the optimization displacement excluded by the frozen parameters, motivating the retention of parameters most likely to undergo substantial adaptation. Since the full optimization displacement is unavailable beforehand, we use the squared initialization gradient as a practical contribution proxy. Experiments show that this strategy preserves the robustness characteristics of the fully optimized curve while substantially reducing curve-training cost.

The main contributions of this work are summarized as follows:
\begin{itemize}
    \item To the best of our knowledge, Robust CurveMoE is the first multi-norm adversarial defense specifically designed for MoE architectures. It transforms the complementary robustness behaviors encoded along a robust connectivity curve into specialized experts and exploits them through input-dependent sparse routing.

    \item We develop a robustness-constrained expert selection strategy that selects complementary specialists from the robust connectivity curve while explicitly controlling their cross-norm vulnerability. By excluding overly specialized curve models with poor worst-case robustness, the resulting expert pool is less susceptible to adversarial routing toward vulnerable experts.

    \item We propose contribution-guided partial curve updating to reduce curve-construction cost. We theoretically bound the objective gap between partial and full curve optimization and derive an initialization-based gradient criterion for selecting influential parameters before curve training.

    \item Experiments on CIFAR-100 and ImageNet-100 with WideResNet and Vision Transformer architectures demonstrate that Robust CurveMoE consistently improves clean, norm-specific, and Union robustness over representative baselines. Extensive ablations further validate the effectiveness of partial updating, selective expertization, and robustness-constrained expert selection.
\end{itemize}

\section{Related Work}

\paragraph{Adversarial and multi-norm robustness}
Deep neural networks are vulnerable to adversarial examples, in which small, carefully crafted perturbations can induce severe prediction errors while remaining nearly imperceptible to humans~\cite{szegedy2013intriguing,goodfellow2014explaining}. Adversarial training remains one of the most effective defenses and is typically formulated as a min--max optimization problem over perturbations within a prescribed norm ball~\cite{madry2017towards}. Its robustness, however, is usually specialized to the threat model used during training and may transfer poorly across perturbation norms. Multi-norm defenses address this limitation by optimizing over the union of multiple threat models. Representative approaches aggregate norm-specific losses, select the strongest attack during training, or train over representative extreme norms~\cite{tramer2019adversarial,maini2020adversarial,croce2022adversarial,maini2022perturbation}. Although these methods improve robustness against multiple attacks, they generally encode competing norm-specific objectives within a single parameter configuration. This can lead to an unfavorable compromise between clean accuracy and robustness across different norms. Moreover, generating adversarial examples separately for multiple threat models introduces considerable training overhead. In contrast, we investigate multi-norm defense from an MoE-oriented perspective, preserving complementary robustness behaviors across different experts and adaptively routing each input to the most suitable expert.

\paragraph{Mixture-of-experts and adversarial robustness}
Mixture-of-experts (MoE) models employ a routing module to activate a sparse subset of experts for each input, increasing model capacity without a proportional increase in computation~\cite{shazeer2017outrageously,fedus2022switch}. Their conditional-computation structure has been widely adopted in large-scale language models and extended to vision architectures~\cite{riquelme2021scaling}. Recent work has also investigated the adversarial robustness of MoEs. Puigcerver et al. show that sparse MoEs can exhibit improved robustness over dense models with comparable computational cost~\cite{puigcerver2022adversarial}. AdvMoE decomposes the robustness of CNN-based MoEs into router and expert robustness and introduces alternating adversarial training to improve both components~\cite{zhang2023robust}. Standard and robust experts have also been combined to alleviate the clean--robust accuracy trade-off~\cite{zhang2025optimizing}. Nevertheless, existing robust MoE approaches primarily consider robustness under a single threat model or combine experts with standard-versus-robust roles. They do not explicitly preserve and route among experts specialized to different perturbation geometries. In addition, directly training multiple independent norm-specific experts would substantially increase expert construction cost. Robust CurveMoE instead derives correlated yet complementary experts from a robust connectivity curve and performs input-dependent sparse routing to achieve multi-norm robustness.

\paragraph{Mode connectivity and robust model composition}
Mode connectivity reveals that independently trained solutions can often be connected by simple low-loss paths in weight space~\cite{garipov2018loss,draxler2018essentially}. This perspective has been used to study optimization stability, the structure of solution basins, and lottery-ticket subnetworks~\cite{frankle2020linear}. Related weight-space composition methods, such as model soups, further show that combining appropriately related models can improve predictive performance without deploying a conventional ensemble~\cite{wortsman2022model}. In adversarial learning, mode connectivity has been used to study robustness barriers between standard and adversarially trained solutions and to search for models exhibiting different robustness profiles against multiple $\ell_p$ attacks~\cite{zhao2020bridging,wang2023robust}. These studies suggest that a robust connectivity curve can provide a structured source of complementary models. However, existing methods generally treat individual curve locations as standalone models rather than organizing them as input-adaptive experts. Furthermore, constructing the curve requires adversarial optimization at sampled locations in weight space, while deploying multiple full curve models incurs substantial storage and inference costs. Robust CurveMoE addresses both limitations by selectively expertizing the influential components of curve-derived models and partially updating the curve parameters, reducing expert deployment and construction costs, respectively.

\section{Robust CurveMoE}
\label{sec:robust_curvemoe}

\begin{figure*}[htbp]
    \centering
    \includegraphics[width=0.98\textwidth]{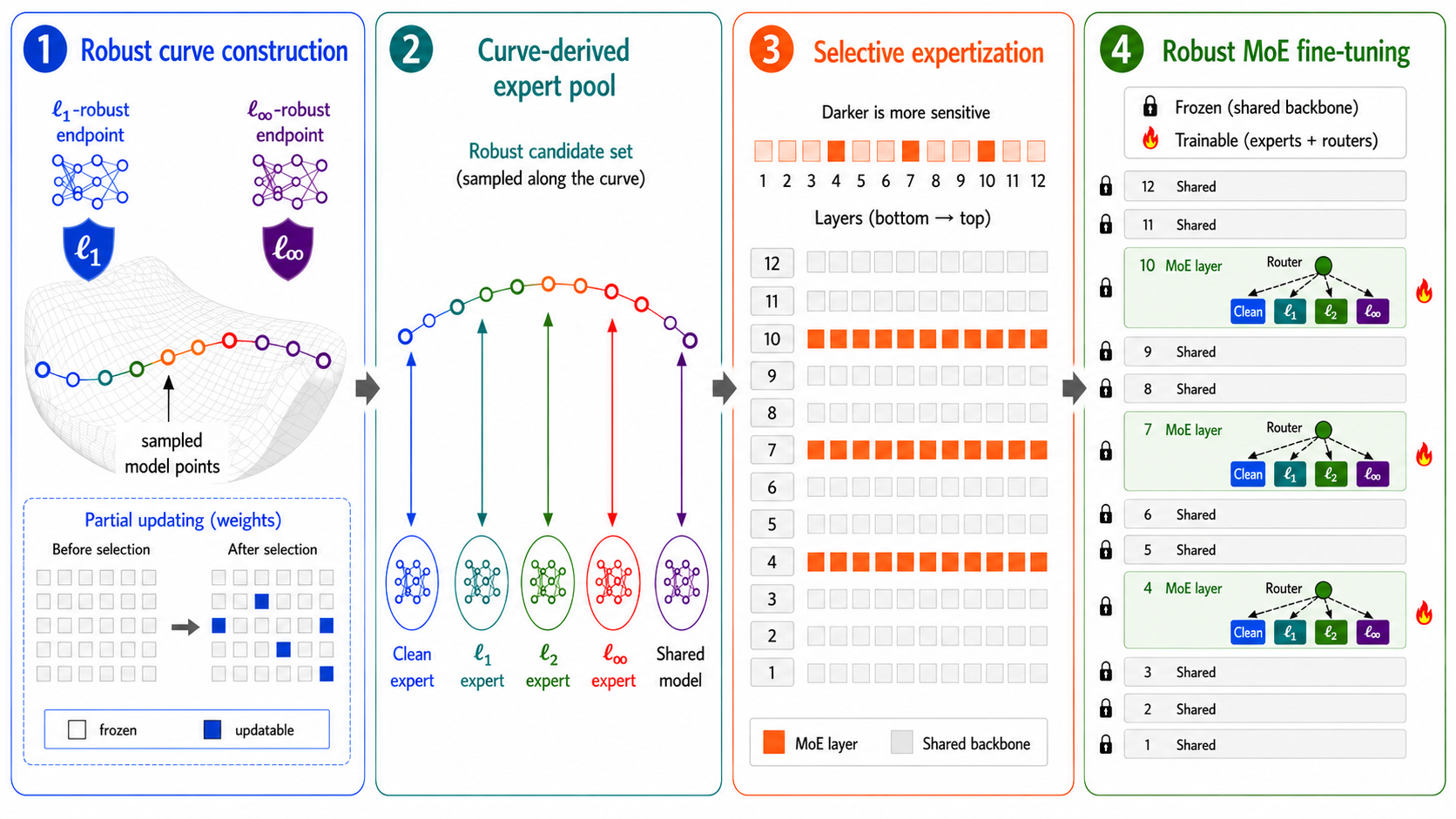}
    \caption{Overview of Robust CurveMoE. Starting from $\ell_1$-robust and $\ell_\infty$-robust endpoint models, we first construct a robust connectivity curve using contribution-guided partial updating. We then sample representative models from the curve to form a robustness-constrained candidate set, from which clean and norm-specialized experts as well as a shared model are selected. Based on the sensitivity of different layers, only the most influential layers are converted into MoE layers, while the remaining layers are shared. Finally, the constructed MoE is fine-tuned by updating the expert and router parameters while keeping the shared backbone frozen.}
    \label{fig:overview}
\end{figure*}

An overview of Robust CurveMoE is shown in Fig.~\ref{fig:overview}. Robust CurveMoE is an efficient multi-norm adversarial defense designed explicitly for MoE architectures. We first train a robust connectivity curve and sample models along the curve to construct a structured pool of experts with complementary robustness profiles. These curve-derived experts are then integrated into an MoE architecture, where input-dependent top-1 routing selects a suitable expert for each input. To retain the computational efficiency of sparse MoE training, we identify the layers most influential to multi-norm robustness and instantiate experts only at the selected layers, while keeping the remaining layers shared. We further introduce contribution-guided partial curve updating to reduce the adversarial optimization cost of constructing the expert pool. Together, these designs enable Robust CurveMoE to provide multi-norm robustness while maintaining efficient expert construction, training, and inference.

\subsection{Preliminaries}
\label{sec:preliminaries}

\paragraph{Multi-norm adversarial defense}
Let $f_{\boldsymbol{\theta}}$ denote a neural network parameterized by $\boldsymbol{\theta}$, and let $(\boldsymbol{x},y)$ be an input-label pair drawn from a data distribution $\mathcal{D}$. An adversarial example is obtained by adding a small perturbation $\boldsymbol{\delta}$ to $\boldsymbol{x}$ so that the perturbed input $\boldsymbol{x}+\boldsymbol{\delta}$ causes an incorrect prediction. The allowable perturbations are commonly constrained by an $\ell_p$ norm as $\|\boldsymbol{\delta}\|_p\leq\epsilon_p$, where $p$ specifies the geometry of the threat model and $\epsilon_p$ controls its strength. Adversarial training improves robustness by minimizing the loss on the strongest perturbation within a prescribed set:
\begin{equation}
\min_{\boldsymbol{\theta}}\ \mathbb{E}_{(\boldsymbol{x},y)\sim\mathcal{D}}\left[\max_{\boldsymbol{\delta}\in\Delta_p}\mathcal{L}\left(f_{\boldsymbol{\theta}}(\boldsymbol{x}+\boldsymbol{\delta}),y\right)\right],
\label{eq:single_norm_at}
\end{equation}
where $\Delta_p=\{\boldsymbol{\delta}:\|\boldsymbol{\delta}\|_p\leq\epsilon_p\}$ and $\mathcal{L}$ is the classification loss. Because a model trained against one perturbation norm may remain vulnerable to others, multi-norm defense considers a collection of threat models $\mathcal{P}$ and seeks robustness over their union:
\begin{equation}
\min_{\boldsymbol{\theta}}\ \mathbb{E}_{(\boldsymbol{x},y)\sim\mathcal{D}}\left[\max_{p\in\mathcal{P}}\ \max_{\boldsymbol{\delta}\in\Delta_p}\mathcal{L}\left(f_{\boldsymbol{\theta}}(\boldsymbol{x}+\boldsymbol{\delta}),y\right)\right].
\label{eq:multi_norm_at}
\end{equation}
In this work, $\mathcal{P}$ contains the $\ell_1$, $\ell_2$, and $\ell_\infty$ threat models. Conventional multi-norm defenses optimize these competing robustness objectives within a single parameter configuration, which may lead to an unfavorable compromise across different norms and requires adversarial examples to be generated under multiple threat models.

\paragraph{Mixture-of-experts architectures}
A mixture-of-experts (MoE) model consists of a set of expert networks and a router that determines which expert should process each input. Given an input representation $\boldsymbol{h}$, the router produces a score for each of the $K$ experts:
\begin{equation}
\boldsymbol{\pi}(\boldsymbol{h})=\operatorname{softmax}\left(g_{\boldsymbol{\phi}}(\boldsymbol{h})\right),
\label{eq:router_probability}
\end{equation}
where $g_{\boldsymbol{\phi}}$ is a learnable routing function parameterized by $\boldsymbol{\phi}$. Under top-1 routing, only the expert with the highest routing score is activated:
\begin{equation}
k^*(\boldsymbol{h})=\arg\max_{k\in\{1,\ldots,K\}}\pi_k(\boldsymbol{h}), \qquad \boldsymbol{h}'=E_{k^*(\boldsymbol{h})}(\boldsymbol{h}),
\label{eq:top1_routing}
\end{equation}
where $E_k$ denotes the $k$-th expert. This conditional-computation mechanism allows an MoE model to increase its parameter capacity and learn specialized behaviors without activating every expert for every input. Consequently, sparse routing avoids a proportional increase in computation as the number of experts grows and enables efficient inference by evaluating only the selected expert. These architectural benefits are largely absent from conventional dense multi-norm defenses, which process every input using the same complete parameter configuration. Robust CurveMoE exploits this structure by preserving complementary robustness behaviors across experts and routing each input to a single suitable expert.

\paragraph{Mode connectivity}
Mode connectivity studies whether two independently trained neural networks can be connected by a continuous path in parameter space while maintaining low loss along the path. Let $\boldsymbol{\theta}_1$ and $\boldsymbol{\theta}_2$ denote two trained endpoint models. A parameterized curve $\boldsymbol{\phi}_{\boldsymbol{\omega}}(t)$, where $t\in[0,1]$, connects them such that $\boldsymbol{\phi}_{\boldsymbol{\omega}}(0)=\boldsymbol{\theta}_1$ and $\boldsymbol{\phi}_{\boldsymbol{\omega}}(1)=\boldsymbol{\theta}_2$. We employ a quadratic B\'ezier curve~\cite{wang2023robust}:
\begin{equation}
\boldsymbol{\phi}_{\boldsymbol{\omega}}(t)=(1-t)^2\boldsymbol{\theta}_1+2t(1-t)\boldsymbol{\omega}+t^2\boldsymbol{\theta}_2,
\label{eq:bezier_curve}
\end{equation}
where the endpoints $\boldsymbol{\theta}_1$ and $\boldsymbol{\theta}_2$ remain fixed and the bend parameter $\boldsymbol{\omega}$ is optimized to obtain low-loss models along the curve. When the two endpoints are trained against different perturbation norms, models sampled from different curve locations can exhibit distinct yet complementary robustness profiles. This provides a structured source of experts for multi-norm defense without requiring every expert to be trained independently from scratch.

\subsection{Curve-Derived Expert Pool for Multi-Norm MoE}
\label{sec:curve_expert_pool}

An effective MoE architecture relies on a collection of experts with complementary capabilities. For multi-norm adversarial defense, a straightforward strategy is to train separate experts against different perturbation norms. However, independently training multiple norm-specific experts requires repeated adversarial optimization and substantially increases the cost of constructing the expert pool. We instead obtain candidate experts from a single robust connectivity curve connecting models trained under different threat models.

We construct the expert pool from a robust connectivity curve between two models specialized to different perturbation norms. Specifically, the endpoint $\boldsymbol{\theta}_1$ is obtained through adversarial training against $\ell_1$ perturbations, while the endpoint $\boldsymbol{\theta}_2$ is trained against $\ell_\infty$ perturbations. These two norms represent geometrically distinct threat models and provide complementary robustness characteristics at the endpoints. We connect them using the quadratic B\'ezier curve defined in Eq.~\eqref{eq:bezier_curve}. During curve training, the endpoint parameters $\boldsymbol{\theta}_1$ and $\boldsymbol{\theta}_2$ remain fixed, and only the bend parameter $\boldsymbol{\omega}$ is optimized.

For each mini-batch, we sample a curve location $t$ from the uniform distribution over $[0,1]$, denoted by $t\sim\mathcal{U}(0,1)$, and instantiate the corresponding model $f_{\boldsymbol{\phi}_{\boldsymbol{\omega}}(t)}$. To encourage robustness throughout the curve rather than only at its endpoints, we optimize the sampled model against the union of the $\ell_1$, $\ell_2$, and $\ell_\infty$ threat models. The robust curve objective is formulated as
\begin{equation}
\min_{\boldsymbol{\omega}}\;
\mathbb{E}_{(\boldsymbol{x},y)\sim\mathcal{D}}
\mathbb{E}_{t\sim\mathcal{U}(0,1)}
\left[
\max_{p\in\mathcal{P}}
\max_{\boldsymbol{\delta}\in\Delta_p}
\mathcal{L}
\left(
f_{\boldsymbol{\phi}_{\boldsymbol{\omega}}(t)}
(\boldsymbol{x}+\boldsymbol{\delta}),
y
\right)
\right],
\label{eq:robust_curve_objective}
\nonumber
\end{equation}
where $\mathcal{P}=\{1,2,\infty\}$. In practice, we approximate the inner maximization using the multi-steepest-descent (MSD) attack~\cite{maini2020adversarial}, which generates norm-specific candidates and selects the strongest candidate for each training sample.

After training, the resulting curve defines a continuous family of robust models. Given a set of representative locations $\mathcal{T}=\{t_1,\ldots,t_K\}$, we construct the candidate expert pool as
\begin{equation}
\mathcal{E}_{\mathrm{curve}}
=
\left\{
f_{\boldsymbol{\phi}_{\boldsymbol{\omega}}(t_k)}
\,\middle|\,
t_k\in\mathcal{T}
\right\}.
\label{eq:curve_expert_pool}
\end{equation}
Figure~\ref{fig:curve_expert_specialization} illustrates the clean and norm-specific adversarial performance of WideResNet-28-10 models sampled along the robust connectivity curve on CIFAR-100. Distinct regions of the curve exhibit different robustness preferences. Models near the $\ell_1$-robust endpoint generally retain stronger $\ell_1$ robustness, whereas models closer to the $\ell_\infty$-robust endpoint achieve better $\ell_\infty$ robustness. In contrast, $\ell_2$ robustness is stronger at intermediate curve locations, forming two local performance peaks away from the endpoints. These observations indicate that different curve regions naturally specialize toward different evaluation objectives rather than representing redundant parameter configurations. The robust connectivity curve therefore provides a structured source of complementary candidate experts for multi-norm MoE construction.

\begin{figure}[htbp]
    \centering
    \vspace{-5pt}
    \includegraphics[width=0.98\columnwidth]{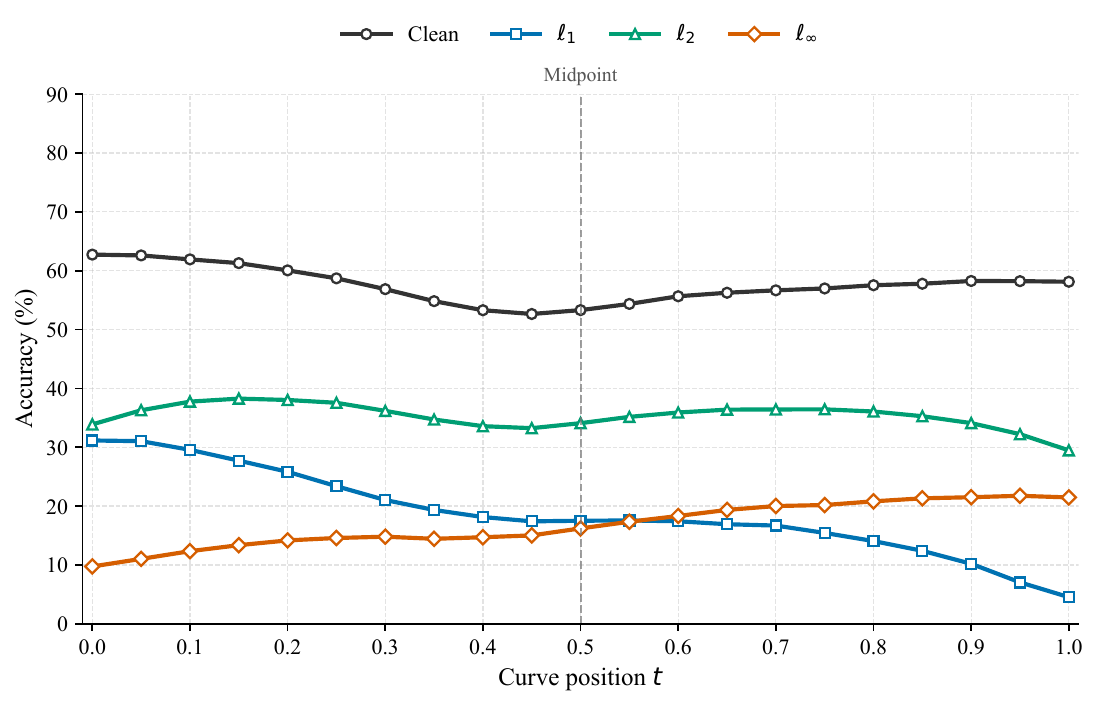}
    \vspace{-5pt}
    \caption{Clean and norm-specific adversarial accuracy along the robust connectivity curve of WideResNet-28-10 on CIFAR-100, evaluated using AutoAttack. Different curve locations exhibit distinct robustness preferences.}
    \label{fig:curve_expert_specialization}
\end{figure}

The complementary models along the curve can be exploited through either ensemble aggregation or an MoE architecture. We argue that an MoE architecture is better suited to preserving and exploiting their distinct robustness profiles. Conventional ensemble strategies such as FGE~\cite{garipov2018loss} aggregate predictions from multiple constituent models for every input, regardless of their individual robustness preferences. For example, under an $\ell_1$ attack, a model near the $\ell_1$-robust endpoint may correctly classify the adversarial input, while models from other curve regions that are less robust to $\ell_1$ perturbations may produce incorrect predictions. Aggregating these predictions can dilute the contribution of the $\ell_1$-robust model and potentially alter the final prediction. In contrast, an MoE architecture employs input-dependent sparse routing to select a suitable expert, thereby avoiding prediction aggregation across experts with different robustness preferences. This conditional routing mechanism better preserves the complementary robustness behaviors observed along the curve and motivates our curve-derived MoE construction.

\subsection{Efficient Robust CurveMoE Construction and Optimization}
\label{sec:efficient_moe_construction}

\paragraph{Sensitivity-guided selective expertization}
A straightforward way to construct an MoE from the curve-derived models is to treat each selected model as a complete expert and employ a router to select among them. Although this design preserves the complementary robustness profiles along the curve, replicating the entire network for every expert substantially increases the number of expert-specific parameters and undermines the parameter efficiency of sparse MoE architectures. We therefore seek to retain expert-specific parameters only in the layers that are most sensitive to the adversarial objectives, while sharing the remaining layers across experts.

To quantify the sensitivity of individual layers, we partition the trainable curve parameters into $M$ disjoint groups, $\boldsymbol{\omega}={\boldsymbol{\omega}^{(1)},\ldots,\boldsymbol{\omega}^{(M)}}$, where each group corresponds to a candidate MoE layer. For ViT, we treat the MLP module in each Transformer block as a candidate layer. Given an evaluation objective $q$, we measure the normalized gradient sensitivity of layer $m$ as
\begin{equation}
\label{eq:moe_layer_sensitivity}
C_m^{(q)}
=
\frac{1}{d_m}
\mathbb{E}_{(\boldsymbol{x},y)\sim\mathcal{D}}
\mathbb{E}_{t\sim\mathcal{U}(0,1)}
\left[
    \left\|
    \nabla_{\boldsymbol{\omega}^{(m)}}
    \widehat{\mathcal{J}}_{q}
    \left(
        \boldsymbol{\omega};
        t,\boldsymbol{x},y
    \right)
    \right\|_2^2
\right],
\nonumber
\end{equation}
where $d_m$ denotes the number of trainable curve parameters in layer $m$, and $\widehat{\mathcal{J}}_{q}$ denotes the mini-batch objective corresponding to clean or norm-specific adversarial evaluation. Division by $d_m$ removes the direct dependence on layer size, such that a larger $C_m^{(q)}$ indicates a stronger response of layer $m$ to objective $q$. For comparison across layers, we report the relative contribution by normalizing $C_m^{(q)}$ over all candidate layers for each objective.

Figure~\ref{fig:layer_sensitivity} shows the relative layer-wise contributions of the candidate MoE layers in ViT-Tiny. The contributions are highly non-uniform across Transformer blocks: a small number of modules exhibit substantially stronger responses, whereas several others contribute considerably less. Moreover, the clean and norm-specific objectives exhibit both shared and objective-dependent sensitivity patterns. For example, some blocks are consistently influential across different objectives, while others show noticeably different contributions between clean and adversarial objectives.

\begin{figure}[htbp]
\centering
\vspace{-5pt}
\includegraphics[width=0.98\columnwidth]{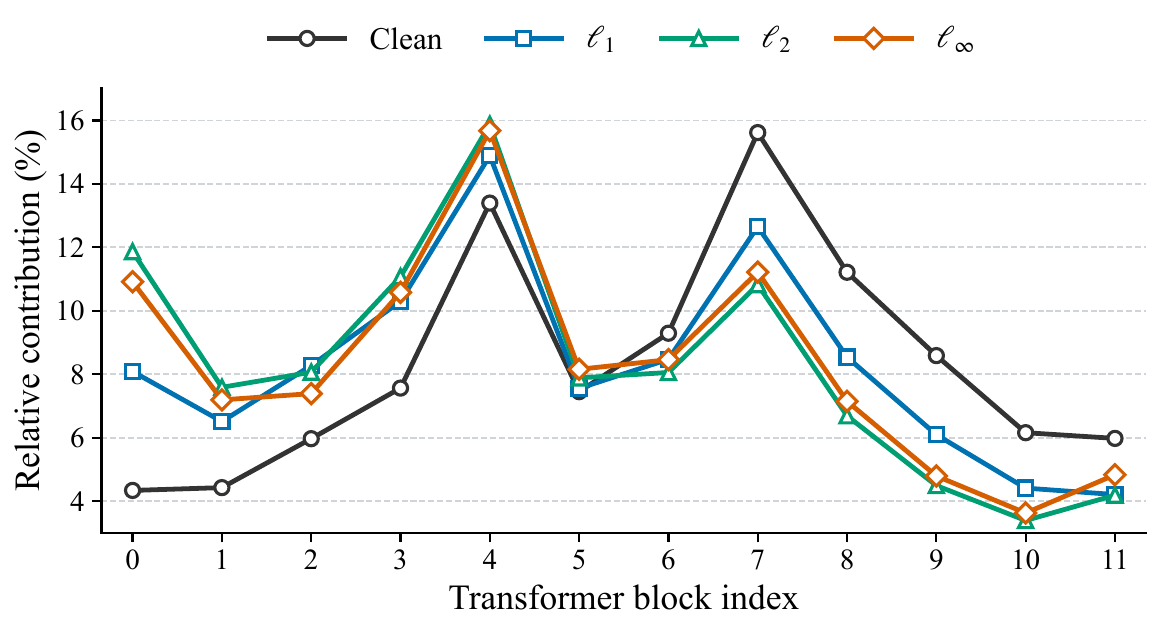}
\vspace{-5pt}
\caption{Relative contribution of the candidate MoE layers in ViT-Tiny to clean and norm-specific adversarial objectives. The layer-wise contributions are highly non-uniform, and different objectives exhibit both shared and objective-dependent sensitivity patterns.}
\label{fig:layer_sensitivity}
\end{figure}

These observations suggest that replicating every layer across experts is unnecessary. Instead, the additional expert capacity can be concentrated on the layers that contribute most strongly to adversarial robustness, while the remaining layers are shared across routing paths. Accordingly, since our objective is to improve multi-norm adversarial robustness, we select the MoE layers based only on their adversarial sensitivity. For each candidate layer, we average its sensitivity scores under the 
$\ell_1$, $\ell_2$, and $\ell_\infty$ objectives:
\begin{equation}
\label{eq:avg_adv_sensitivity}
C_m
=
\frac{1}{|\mathcal{P}|}
\sum_{p\in\mathcal{P}}
C_m^{(p)},
\qquad
\mathcal{P}=\{1,2,\infty\}.
\end{equation}
We then rank the candidate layers according to $C_m$ and select the top-$K$ 
layers for expertization:
\begin{equation}
\label{eq:moe_layer_selection}
\mathcal{S}_{\mathrm{MoE}}
=
\operatorname{TopK}
\left(
\{C_m\}_{m=1}^{M},
K
\right).
\end{equation}
After determining the MoE layers, we next select the curve models from which the expert-specific and shared parameters are extracted. Directly searching the entire curve for the model that maximizes each individual robustness objective may result in highly specialized but vulnerable experts. Prior work has shown that, in adversarial MoE architectures, the experts can constitute a more vulnerable component than the router~\cite{zhang2025optimizing}. This issue is particularly important in our setting because the norm-specialized endpoints may achieve strong robustness against their corresponding threat models while remaining considerably less robust to other perturbation norms. For example, under an $\ell_\infty$ attack, an adversary may perturb the routing decision toward an expert specialized primarily for $\ell_1$ robustness; if this expert has weak $\ell_\infty$ robustness, such misrouting can substantially degrade the final prediction. We therefore restrict both expert and shared-model selection to curve locations that maintain sufficiently strong robustness across all considered threat models.

Specifically, we uniformly discretize the trained curve into a set of sampled locations $\mathcal{T}$ and define the overall robust accuracy of a model at location $t$ as its worst accuracy among the three adversarial objectives:
\begin{equation}
\label{eq:overall_robust_accuracy}
A_{\mathrm{rob}}^{\mathrm{val}}(t)
=
\min_{p \in \mathcal{P}}
A_p^{\mathrm{val}}(t),
\qquad
\mathcal{P} = \{1,2,\infty\},
\end{equation}
where $A_p^{\mathrm{val}}(t)$ denotes the validation accuracy at curve location $t$ under the $\ell_p$ threat model. We then construct a robustness-constrained candidate set by retaining locations whose overall robust accuracy is within a tolerance $\tau_{\mathrm{cand}}$ of the best sampled model:
\begin{equation}
\label{eq:curve_candidate_set}
\mathcal{T}_{\mathrm{cand}}
=
\left\{
t \in \mathcal{T}
\;\middle|\;
A_{\mathrm{rob}}^{\mathrm{val}}(t)
\geq
\max_{t' \in \mathcal{T}}
A_{\mathrm{rob}}^{\mathrm{val}}(t')
-
\tau_{\mathrm{cand}}
\right\},
\end{equation}
where $\tau_{\mathrm{cand}}\geq0$ is a tolerance hyperparameter controlling the robustness range of the candidate set. A smaller $\tau_{\mathrm{cand}}$ imposes a stricter multi-norm robustness requirement, whereas a larger value admits a broader range of curve models with potentially stronger specialization. This constraint prevents models with severely imbalanced robustness profiles from being used as experts while retaining sufficiently robust curve locations for subsequent expert and shared-model selection.

Within $\mathcal{T}_{\mathrm{cand}}$, we select four specialist models according to their clean and norm-specific validation performance:
\begin{equation}
\label{eq:expert_location_selection}
t_q^{*}
=
\arg\max_{t \in \mathcal{T}_{\mathrm{cand}}}
A_q^{\mathrm{val}}(t),
\qquad
q \in
\left\{
\mathrm{std},
\ell_1,
\ell_2,
\ell_\infty
\right\}.
\end{equation}
Although our primary objective is multi-norm adversarial robustness, we additionally retain the clean-accuracy specialist to exploit the conditional capacity of the MoE architecture and improve standard performance without discarding the norm-specific robust experts. For every selected MoE layer $m\in\mathcal{S}_{\mathrm{MoE}}$, the corresponding parameters from these four specialist models are instantiated as four expert branches.

The remaining layers are shared by all routing paths and are therefore initialized from a common base model. Since all models in $\mathcal{T}_{\mathrm{cand}}$ already satisfy the multi-norm robustness constraint, the shared base model is selected according to its parameter-space compatibility with the specialist models. Let $\boldsymbol{\theta}(t)$ denote the complete parameter vector of the model at curve location $t$. We measure the parameter-space compatibility between a candidate base model and the selected specialist models using their average normalized distance:
\begin{equation}
\label{eq:base_parameter_distance}
D_{\mathrm{base}}(t)
=
\frac{1}{|\mathcal{Q}|}
\sum_{q \in \mathcal{Q}}
\frac{\left\|\boldsymbol{\theta}(t)-\boldsymbol{\theta}(t_q^{*})\right\|_2^2}{\left\|\boldsymbol{\theta}(t_q^{*})\right\|_2^2+\varepsilon},
\end{equation}
where $\mathcal{Q}=\{\mathrm{std},\ell_1,\ell_2,\ell_\infty\}$ and $\varepsilon>0$ is a small constant for numerical stability. A smaller $D_{\mathrm{base}}(t)$ indicates that the candidate model is more compatible with the specialist models in the overall parameter space. The shared base model is then selected from the robustness-constrained candidate set as
\begin{equation}
\label{eq:shared_base_selection}
t_{\mathrm{base}}^{*}
=
\arg\min_{t \in \mathcal{T}_{\mathrm{cand}}}D_{\mathrm{base}}(t).
\end{equation}
The layers outside $\mathcal{S}_{\mathrm{MoE}}$ are inherited from the model at $t_{\mathrm{base}}^{*}$ and shared across all routing paths. In this way, the resulting architecture combines robustness-constrained specialist experts with a parameter-compatible shared backbone, while introducing expert-specific parameters only at the robustness-sensitive layers.

\paragraph{Contribution-guided partial curve updating}
Constructing the robust connectivity curve remains computationally expensive because, after training the two norm-specialized endpoints, the curve parameters must be further optimized using multi-norm adversarial training. Motivated by the highly non-uniform layer-wise sensitivity observed above, we investigate whether a comparable robust curve can be obtained by updating only a subset of the trainable curve parameters. Specifically, starting from the initial curve parameter $\boldsymbol{\omega}_0$, we optimize only a selected subset of parameter groups while keeping the remaining groups fixed at their initial values throughout curve training.

We first examine random partial updating, where a prescribed fraction of parameter groups is randomly selected for optimization. As shown in Table~\ref{tab:random_partial_update}, partial updating consistently reduces curve-training time, but random parameter selection leads to noticeable performance degradation and increased variation, especially at lower updating ratios. For example, when the updating ratio is reduced to $30\%$, the degradation becomes more pronounced across clean and adversarial objectives, with particularly large drops under the $\ell_2$ and $\ell_\infty$ threat models. These results indicate that different parameter groups contribute unequally to robust curve optimization, and randomly excluding influential groups can substantially impair both the effectiveness and stability of the resulting curve. Therefore, an effective partial-updating strategy should identify the parameter groups that are most important to the multi-norm robust objective before curve optimization begins.

\begin{table}[htbp]
\centering
\caption{Efficiency and robustness of curve updating strategies.}
\label{tab:random_partial_update}
\resizebox{\columnwidth}{!}{
\begin{tabular}{lcccccc}
\toprule
Updating Strategy & Ratio & Time Reduction (\%) & Clean (\%) & $\ell_1$ (\%) & $\ell_2$ (\%) & $\ell_\infty$ (\%) \\
\midrule
Full Updating       & 100\% & $0$              & $57.52 \pm 2.99$ & $18.78 \pm 7.35$ & $35.29 \pm 2.12$ & $16.78 \pm 3.74$ \\
Random              & 50\%  & $35.26 \pm 0.57$  & $56.32 \pm 4.97$ & $16.85 \pm 9.57$ & $32.83 \pm 4.85$ & $15.91 \pm 4.21$ \\
Random              & 30\%  & $36.56 \pm 0.31$  & $52.95 \pm 7.92$ & $17.93 \pm 8.49$ & $30.75 \pm 6.78$ & $12.39 \pm 7.43$ \\
Contribution-Guided & 50\%  & $35.14 \pm 0.70$  & $57.13 \pm 3.41$ & $18.54 \pm 7.45$ & $34.98 \pm 2.20$ & $16.58 \pm 3.86$ \\
Contribution-Guided & 30\%  & $36.05 \pm 0.82$  & $56.82 \pm 3.73$ & $18.33 \pm 7.56$ & $34.67 \pm 2.36$ & $16.41 \pm 3.96$ \\
\bottomrule
\end{tabular}
}
\end{table}

To characterize the effect of restricting curve optimization to a subset of parameters, let $\boldsymbol{\omega}_0$ denote the initialization of the trainable curve parameters and let
\begin{equation}
\label{eq:full_curve_solution}
\boldsymbol{\omega}^{*}
=
\boldsymbol{\omega}_0
+
\boldsymbol{\Delta}^{*},
\end{equation}
where $\boldsymbol{\omega}^{*}$ denotes a stationary solution obtained through full curve optimization and $\boldsymbol{\Delta}^{*}$ represents the corresponding optimization displacement. For a subset of parameters $\mathcal{S}$, let $\boldsymbol{P}_{\mathcal{S}}$ denote the projection onto the parameter subspace specified by $\mathcal{S}$. The feasible set of the corresponding partial-update problem is
\begin{equation}
\label{eq:restricted_feasible_set}
\mathcal{F}_{\mathcal{S}}
=
\left\{
\boldsymbol{\omega}_0
+
\boldsymbol{P}_{\mathcal{S}}\boldsymbol{d}
:
\boldsymbol{d}\in\mathbb{R}^{d}
\right\},
\end{equation}
and the restricted solution is defined as
\begin{equation}
\label{eq:restricted_curve_solution}
\boldsymbol{\omega}_{\mathcal{S}}^{*}
\in
\arg\min_{\boldsymbol{\omega}\in\mathcal{F}_{\mathcal{S}}}
\mathcal{J}(\boldsymbol{\omega}),
\end{equation}
where $\mathcal{J}$ denotes the curve objective.

\begin{theorem}[Robust-objective preservation under restricted parameters]
\label{thm:partial_update}
Suppose that $\mathcal{J}$ is $L$-smooth in a neighborhood containing $\boldsymbol{\omega}^{*}$ and the restricted solutions, and that $\boldsymbol{\omega}^{*}$ is a stationary point of the full curve optimization problem. Then,
\begin{equation}
\label{eq:partial_update_bound}
\mathcal{J}\left(\boldsymbol{\omega}_{\mathcal{S}}^{*}\right)
-
\mathcal{J}\left(\boldsymbol{\omega}^{*}\right)
\leq
\frac{L}{2}
\left\|
\boldsymbol{P}_{\mathcal{S}}^{\perp}
\boldsymbol{\Delta}^{*}
\right\|_2^2,
\end{equation}
where $\boldsymbol{P}_{\mathcal{S}}^{\perp}
=
\boldsymbol{I}
-
\boldsymbol{P}_{\mathcal{S}}$.
\end{theorem}

The proof is provided in Appendix~\ref{app:partial_update_proof}. Theorem~\ref{thm:partial_update} bounds the objective gap between full curve updating and partial curve updating by $\frac{L}{2}\left\|\boldsymbol{P}_{\mathcal{S}}^{\perp}\boldsymbol{\Delta}^{*}\right\|_2^2$. Since $L$ is fixed for the considered objective, making the partially updated curve perform as closely as possible to the fully updated curve requires minimizing the excluded optimization displacement $\left\|\boldsymbol{P}_{\mathcal{S}}^{\perp}\boldsymbol{\Delta}^{*}\right\|_2^2$. Therefore, the selected parameter subset $\mathcal{S}$ should retain the parameters that undergo the largest changes during full optimization. However, the corresponding displacement $\boldsymbol{\Delta}^{*}$ is only available after completing full curve training, which defeats the purpose of partial updating.

To obtain a selection criterion before curve optimization begins, we approximate the importance of each parameter using its local gradient at the initial curve configuration $\boldsymbol{\omega}_0$. For parameter $\omega_m$, we define the initialization-based contribution score as
\begin{equation}
\label{eq:partial_update_contribution}
G_m
=
\mathbb{E}_{(\boldsymbol{x},y)\sim\mathcal{D}}
\mathbb{E}_{t\sim\mathcal{U}(0,1)}
\left[
\left(
\frac{\partial
\widehat{\mathcal{J}}_{\mathrm{MSD}}
\left(
\boldsymbol{\omega}_0;
t,\boldsymbol{x},y
\right)}
{\partial \omega_m}
\right)^2
\right].
\nonumber
\end{equation}
Here, $\widehat{\mathcal{J}}_{\mathrm{MSD}}$ denotes the mini-batch multi-norm adversarial objective approximated using MSD. Since gradient-based optimization initially changes each parameter according to its local gradient at $\boldsymbol{\omega}_0$, a larger $G_m$ indicates that parameter $\omega_m$ is more strongly driven to adapt under the multi-norm robust objective. We therefore use $G_m$ as a practical proxy for the parameter displacement characterized in Theorem~\ref{thm:partial_update}.

Given an updating ratio $\rho\in(0,1]$, we rank all trainable curve parameters according to their contribution scores and select the top $\rho$ fraction for optimization:
\begin{equation}
\label{eq:partial_update_selection}
\mathcal{S}_{\mathrm{upd}}
=
\operatorname{TopK}
\left(
\{G_m\}_{m=1}^{M},
\left\lceil \rho M \right\rceil
\right),
\end{equation}
where $M$ denotes the total number of trainable curve parameters. Starting from $\boldsymbol{\omega}_0$, only the parameters indexed by $\mathcal{S}_{\mathrm{upd}}$ are optimized during subsequent robust curve training, while all remaining parameters are kept fixed at their initial values. In this way, as demonstrated in Table~\ref{tab:random_partial_update}, contribution-guided partial updating concentrates the optimization budget on the parameters predicted to be most influential to the curve objective, reducing the cost of curve construction while avoiding the instability introduced by random parameter selection.

\paragraph{Robust MoE fine-tuning}
After selective expertization, the resulting MoE combines shared parameters from the base model with expert-specific parameters extracted from different specialist locations along the curve. Although each component originates from a well-performing curve model, their recombination does not correspond to any model directly optimized along the original curve and may therefore introduce representation mismatch between the shared backbone and the inserted experts. Moreover, the newly introduced router must adapt to the heterogeneous robustness profiles of the curve-derived experts. We therefore perform a joint robust fine-tuning stage to adapt the constructed MoE as a unified model.

Specifically, we fine-tune only the expert parameters in the selected MoE layers together with the router, while keeping all shared parameters fixed. The optimization objective is a weighted combination of the standard loss and the multi-norm adversarial loss:
\begin{equation}
\label{eq:moe_finetune_objective}
\min_{\boldsymbol{\Theta}_{\mathrm{exp}},\boldsymbol{\Phi}}
\;
(1-\alpha)\mathcal{L}_{\mathrm{std}}
+
\alpha\mathcal{L}_{\mathrm{rob}},
\end{equation}
where $\boldsymbol{\Theta}_{\mathrm{exp}}$ denotes the parameters of the curve-derived experts, $\boldsymbol{\Phi}=\{\boldsymbol{\phi}^{(m)}\}_{m\in\mathcal{S}_{\mathrm{MoE}}}$ denotes the collection of router parameters associated with the selected MoE layers and $\alpha\in[0,1]$ controls the trade-off between standard accuracy and multi-norm adversarial robustness. Here, $\mathcal{L}_{\mathrm{std}}$ denotes the classification loss on clean inputs, while $\mathcal{L}_{\mathrm{rob}}$ denotes the MSD-approximated worst-case loss over $\mathcal{P}=\{1,2,\infty\}$. Unlike commonly used MoE training practices, we do not employ a separate router warm-up stage or an auxiliary load-balancing loss, as neither provides additional performance gains in our experiments; detailed comparisons are provided in the experimental section. This single-stage fine-tuning adapts the curve-derived experts and router to the shared backbone while preserving the parameters of the selected base model. The complete training and construction procedure of Robust CurveMoE is summarized in Algorithm~\ref{alg:robust_curvemoe}.

\begin{algorithm}[htbp]
\caption{Robust CurveMoE Training and Construction}
\label{alg:robust_curvemoe}
\begin{algorithmic}[1]
\REQUIRE Training data $\mathcal{D}$; threat models $\mathcal{P}=\{1,2,\infty\}$; evaluation objectives $\mathcal{Q}=\{\mathrm{std},\ell_1,\ell_2,\ell_\infty\}$; updating ratio $\rho$; number of MoE layers $K$; candidate tolerance $\tau_{\mathrm{cand}}$; fine-tuning weight $\alpha$
\ENSURE Robust CurveMoE model

\STATE Train $\ell_1$- and $\ell_\infty$-robust endpoints $\boldsymbol{\theta}_1$ and $\boldsymbol{\theta}_2$.
\STATE Initialize the curve parameter $\boldsymbol{\omega}_0$.

\FOR{each trainable curve parameter $\omega_m$}
\STATE $G_m \leftarrow
\mathbb{E}_{(\boldsymbol{x},y)\sim\mathcal{D}}
\mathbb{E}_{t\sim\mathcal{U}(0,1)}
\left[
\left(
\frac{\partial \widehat{\mathcal{J}}_{\mathrm{MSD}}(\boldsymbol{\omega}_0;t,\boldsymbol{x},y)}
{\partial \omega_m}
\right)^2
\right]$.
\ENDFOR

\STATE $\mathcal{S}_{\mathrm{upd}}
\leftarrow
\operatorname{TopK}
\left(
\{G_m\}_{m=1}^{M},
\left\lceil \rho M \right\rceil
\right)$.
\STATE Optimize only the parameters indexed by $\mathcal{S}_{\mathrm{upd}}$ under the MSD-based multi-norm curve objective; keep all remaining curve parameters fixed.

\STATE Uniformly sample curve locations $\mathcal{T}\subset[0,1]$.
\FOR{each $t\in\mathcal{T}$}
\STATE Evaluate $A_q^{\mathrm{val}}(t)$ for all $q\in\mathcal{Q}$.
\STATE $A_{\mathrm{rob}}^{\mathrm{val}}(t)
\leftarrow
\min_{p\in\mathcal{P}}A_p^{\mathrm{val}}(t)$.
\ENDFOR

\STATE $\mathcal{T}_{\mathrm{cand}}
\leftarrow
\left\{
t\in\mathcal{T}:
A_{\mathrm{rob}}^{\mathrm{val}}(t)
\geq
\max_{t'\in\mathcal{T}}
A_{\mathrm{rob}}^{\mathrm{val}}(t')
-
\tau_{\mathrm{cand}}
\right\}$.

\FOR{each $q\in\mathcal{Q}$}
\STATE $t_q^{*}
\leftarrow
\arg\max_{t\in\mathcal{T}_{\mathrm{cand}}}
A_q^{\mathrm{val}}(t)$.
\ENDFOR

\FOR{each candidate MoE layer $m$}
\FOR{each $q\in\mathcal{Q}$}
\STATE Compute the trained-curve sensitivity $C_m^{(q)}$.
\ENDFOR
\STATE $C_m
\leftarrow
\frac{1}{3}
\left(
C_m^{(\ell_1)}
+
C_m^{(\ell_2)}
+
C_m^{(\ell_\infty)}
\right)$.
\ENDFOR

\STATE $\mathcal{S}_{\mathrm{MoE}}
\leftarrow
\operatorname{TopK}
\left(
\{C_m\},
K
\right)$.

\FOR{each $t\in\mathcal{T}_{\mathrm{cand}}$}
\STATE $D_{\mathrm{base}}(t)
\leftarrow
\frac{1}{|\mathcal{Q}|}
\sum_{q\in\mathcal{Q}}
\frac{
\left\|
\boldsymbol{\theta}(t)
-
\boldsymbol{\theta}(t_q^{*})
\right\|_2^2
}{
\left\|
\boldsymbol{\theta}(t_q^{*})
\right\|_2^2
+
\varepsilon
}$.
\ENDFOR

\STATE $t_{\mathrm{base}}^{*}
\leftarrow
\arg\min_{t\in\mathcal{T}_{\mathrm{cand}}}
D_{\mathrm{base}}(t)$.

\STATE Construct the shared backbone from $\boldsymbol{\theta}(t_{\mathrm{base}}^{*})$.
\FOR{each $m\in\mathcal{S}_{\mathrm{MoE}}$}
\STATE Instantiate four experts from
$\{\boldsymbol{\theta}^{(m)}(t_q^{*})\}_{q\in\mathcal{Q}}$.
\STATE Initialize an independent router $\boldsymbol{\phi}^{(m)}$.
\ENDFOR

\STATE $\boldsymbol{\Phi}
\leftarrow
\{\boldsymbol{\phi}^{(m)}\}_{m\in\mathcal{S}_{\mathrm{MoE}}}$.
\STATE Freeze all shared parameters.
\STATE $(\boldsymbol{\Theta}_{\mathrm{exp}}^{*},\boldsymbol{\Phi}^{*})
\leftarrow
\arg\min_{\boldsymbol{\Theta}_{\mathrm{exp}},\boldsymbol{\Phi}}
\left[
(1-\alpha)\mathcal{L}_{\mathrm{std}}
+
\alpha\mathcal{L}_{\mathrm{rob}}
\right]$.

\RETURN Fine-tuned Robust CurveMoE.
\end{algorithmic}
\end{algorithm}

\section{Experimental Evaluation}

\subsection{Experimental Setup}
\label{sec:experimental_setup}

\paragraph{Datasets and architectures}
We conduct experiments on CIFAR-100~\cite{krizhevsky2009learning} and ImageNet-100~\cite{deng2009imagenet} using WideResNet-28-10~\cite{zagoruyko2016wide} and ViT-Tiny/16~\cite{touvron2021training}, respectively. These two settings cover both convolutional and Transformer-based architectures and are used to evaluate the generality of Robust CurveMoE across model families and datasets.

\paragraph{Threat models and evaluation}
We consider $\ell_1$, $\ell_2$, and $\ell_\infty$ adversarial perturbations throughout the experiments. For CIFAR-100, the perturbation budgets are $\epsilon_1=12$, $\epsilon_2=0.5$, and $\epsilon_\infty=8/255$. For ImageNet-100, we use $\epsilon_1=75$, $\epsilon_2=2.0$, and $\epsilon_\infty=4/255$. Robust accuracy is evaluated separately under the three threat models using AutoAttack~\cite{croce2020reliable}, and standard accuracy is reported on clean test samples.

\paragraph{Training settings}
For each dataset, we first train an $\ell_1$-robust endpoint and an $\ell_\infty$-robust endpoint, followed by robust curve training and MoE fine-tuning. On CIFAR-100, the WideResNet-28-10 endpoints are trained for 150 epochs, the robust connectivity curve for 50 epochs, and the resulting MoE for 50 epochs. On ImageNet-100, the pre-trained ViT-Tiny/16 endpoints are trained for 30 epochs, followed by 20 epochs of curve training and 20 epochs of MoE fine-tuning.

\paragraph{Baselines and hyperparameters}
We compare Robust CurveMoE with MSD~\cite{maini2020adversarial} and ERMC~\cite{wang2023robust}. Unless otherwise specified, we use a partial curve updating ratio of $\rho=30\%$, select $K=3$ MoE layers, set the candidate tolerance to $\tau_{\mathrm{cand}}=2$ percentage points, and use $\alpha=0.6$ to balance the standard and multi-norm adversarial objectives.

\subsection{Overall Multi-Norm Robustness}

Table~\ref{tab:overall_robustness} compares Robust CurveMoE with ERMC and MSD on CIFAR-100 and ImageNet-100. In addition to clean and norm-specific adversarial accuracy, we report \emph{Union accuracy}, which measures the fraction of samples that remain correctly classified under all three considered threat models, i.e., $\ell_1$, $\ell_2$, and $\ell_\infty$. This metric therefore provides a stricter measure of multi-norm robustness by requiring each sample to withstand the union of all perturbation sets.

\begin{table}[htbp]
\centering
\vspace{-5pt}
\caption{Overall clean and multi-norm adversarial robustness.}
\label{tab:overall_robustness}
\resizebox{\columnwidth}{!}{
\begin{tabular}{llccccc}
\toprule
Dataset & Method & Clean (\%) & $\ell_1$ (\%) & $\ell_2$ (\%) & $\ell_\infty$ (\%) & Union (\%) \\
\midrule
\multirow{3}{*}{CIFAR-100}
& ERMC
& $55.65 \pm 0.31$ & $17.42 \pm 0.95$ & $35.90 \pm 0.34$
& $18.32 \pm 0.43$ & $15.86 \pm 0.38$ \\
& MSD
& $54.72 \pm 0.76$ & $24.85 \pm 0.79$ & $33.16 \pm 0.19$
& $17.63 \pm 0.49$ & $15.21 \pm 0.45$ \\
& Robust CurveMoE
& $58.34 \pm 0.64$ & $28.40 \pm 0.70$ & $36.48 \pm 0.75$
& $19.71 \pm 0.27$ & $18.23 \pm 0.67$ \\
\midrule
\multirow{3}{*}{ImageNet-100}
& ERMC
& $67.22 \pm 0.65$ & $42.18 \pm 0.16$ & $50.68 \pm 0.12$
& $40.46 \pm 0.50$ & $37.52 \pm 0.96$ \\
& MSD
& $69.90 \pm 0.34$ & $62.10 \pm 0.58$ & $49.80 \pm 0.23$
& $38.20 \pm 0.75$ & $35.84 \pm 0.26$ \\
& Robust CurveMoE
& $73.58 \pm 0.50$ & $67.62 \pm 0.69$ & $54.34 \pm 0.89$
& $41.96 \pm 0.95$ & $39.65 \pm 0.54$ \\
\bottomrule
\end{tabular}
}
\vspace{-5pt}
\end{table}

Robust CurveMoE consistently achieves the best performance across all evaluation metrics on both datasets. On CIFAR-100, it improves Union accuracy from $15.86\%$ for ERMC and $15.21\%$ for MSD to $18.23\%$, while simultaneously achieving higher clean and norm-specific robustness. In particular, its $\ell_1$ accuracy reaches $28.40\%$, substantially exceeding both baselines, without sacrificing $\ell_2$ or $\ell_\infty$ robustness. Similar improvements are observed on ImageNet-100, where Robust CurveMoE achieves a Union accuracy of $39.65\%$, compared with $37.52\%$ for ERMC and $35.84\%$ for MSD. It also improves clean accuracy to $73.58\%$ and achieves the highest robustness under all three perturbation norms. These results demonstrate that exploiting complementary curve-derived experts through input-dependent routing provides a more favorable balance between standard accuracy and multi-norm robustness.

\subsection{Robustness of Expert Selection}

The candidate tolerance $\tau_{\mathrm{cand}}$ controls how aggressively the expert pool is restricted according to cross-norm robustness. A larger tolerance admits more specialized curve models, but may also introduce experts that are highly vulnerable to perturbation norms different from their specialization. This issue is particularly relevant for adversarial MoE models because an attacker may jointly manipulate both the routing decision and the prediction of the selected expert. We therefore evaluate the robustness of expert selection under a targeted expert attack designed to expose this failure mode.

Specifically, we consider an $\ell_\infty$-bounded attack that targets the $\ell_1$-specialized expert. The adversary simultaneously encourages the router to select the $\ell_1$ expert and causes the selected expert to misclassify the input. We therefore define the attack objective as
\begin{equation}
\label{eq:target_expert_attack}
\mathcal{L}_{\mathrm{target}}
=
\mathcal{L}_{\mathrm{cls}}
+
\lambda_{\mathrm{rt}}
\mathcal{L}_{\mathrm{rt}}.
\end{equation}
Here, $\mathcal{L}_{\mathrm{cls}}$ increases the classification loss of the targeted $\ell_1$ expert, $\mathcal{L}_{\mathrm{rt}}$ encourages the router to select that expert, and $\lambda_{\mathrm{rt}}$ controls the relative strength of the routing objective. We set $\lambda_{\mathrm{rt}} = 0.8$ in this evaluation.

\begin{table}[htbp]
\centering
\vspace{-5pt}
\caption{Effect of candidate tolerance on expert robustness.}
\label{tab:target_expert_attack}
\resizebox{\columnwidth}{!}{
\begin{tabular}{lcccccc}
\toprule
Dataset & $\tau_{\mathrm{cand}}$ & Clean (\%) & $\ell_1$ (\%) & $\ell_2$ (\%) & $\ell_\infty$ (\%) & Targeted Expert (\%) \\
\midrule
\multirow{3}{*}{CIFAR-100}
& $2$  & $58.34 \pm 0.64$ & $28.40 \pm 0.70$ & $36.48 \pm 0.75$ & $19.71 \pm 0.27$ & $14.58 \pm 0.15$ \\
& $5$  & $61.28 \pm 0.33$ & $31.15 \pm 0.31$ & $36.23 \pm 0.65$ & $21.45 \pm 0.83$ & $9.96 \pm 0.78$ \\
& $10$ & $62.73 \pm 0.62$ & $31.65 \pm 0.61$ & $36.68 \pm 0.74$ & $22.29 \pm 0.91$ & $9.74 \pm 0.82$ \\
\midrule
\multirow{3}{*}{ImageNet-100}
& $2$  & $73.58 \pm 0.50$ & $67.62 \pm 0.69$ & $54.34 \pm 0.89$ & $41.96 \pm 0.95$ & $38.18 \pm 0.19$ \\
& $5$  & $73.86 \pm 0.14$ & $68.14 \pm 0.24$ & $54.53 \pm 0.97$ & $42.85 \pm 0.75$ & $31.79 \pm 0.56$ \\
& $10$ & $73.84 \pm 0.25$ & $69.29 \pm 0.35$ & $54.91 \pm 0.72$ & $47.53 \pm 0.38$ & $15.67 \pm 0.59$ \\
\bottomrule
\end{tabular}
}
\vspace{-5pt}
\end{table}

As shown in Table~\ref{tab:target_expert_attack}, increasing $\tau_{\mathrm{cand}}$ generally preserves or improves clean and norm-specific adversarial accuracy, but substantially weakens robustness against the targeted expert attack. On CIFAR-100, the targeted-expert accuracy decreases from $14.58\%$ at $\tau_{\mathrm{cand}}=2$ to $9.74\%$ at $\tau_{\mathrm{cand}}=10$, despite the improvement in clean and individual-norm robustness. A similar but more pronounced trend is observed on ImageNet-100, where the targeted-expert accuracy drops from $38.18\%$ to $15.67\%$ as $\tau_{\mathrm{cand}}$ increases from $2$ to $10$. This indicates that a looser candidate constraint can admit more specialized but cross-norm-vulnerable experts, which may not be exposed by standard norm-specific evaluation but can be exploited through joint routing and classification attacks. These results validate the role of the candidate constraint in improving expert-level robustness and support the use of $\tau_{\mathrm{cand}}=2$.

\subsection{Ablation Study}

\paragraph{Effect of partial updating ratio}
We first investigate the effect of the partial curve updating ratio $\rho$ by comparing contribution-guided updating with $\rho\in\{30\%,50\%,100\%\}$. Figure~\ref{fig:partial_updating_ratio} reports the clean and norm-specific adversarial accuracy along the resulting connectivity curves. Across both WideResNet-28-10 on CIFAR-100 and ViT-Tiny on ImageNet-100, reducing the updating ratio has only a minor effect on the curve performance. In particular, the curves obtained with $30\%$ and $50\%$ updating closely follow those of full updating across clean, $\ell_1$, $\ell_2$, and $\ell_\infty$ evaluations. The agreement is especially strong on ImageNet-100, while only small deviations are observed at some intermediate curve locations on CIFAR-100. More importantly, partial updating preserves the characteristic robustness specialization along the curve, including stronger $\ell_1$ robustness near one endpoint and stronger $\ell_\infty$ robustness toward the other. Together with the training-time reduction reported in Table~\ref{tab:random_partial_update}, these results show that contribution-guided updating can substantially reduce curve optimization cost without materially altering the robustness characteristics of the learned connectivity path. We therefore use $\rho=30\%$ as the default setting in our experiments.

\begin{figure}[htbp]
    \centering
    \vspace{-5pt}
    \includegraphics[width=0.98\columnwidth]{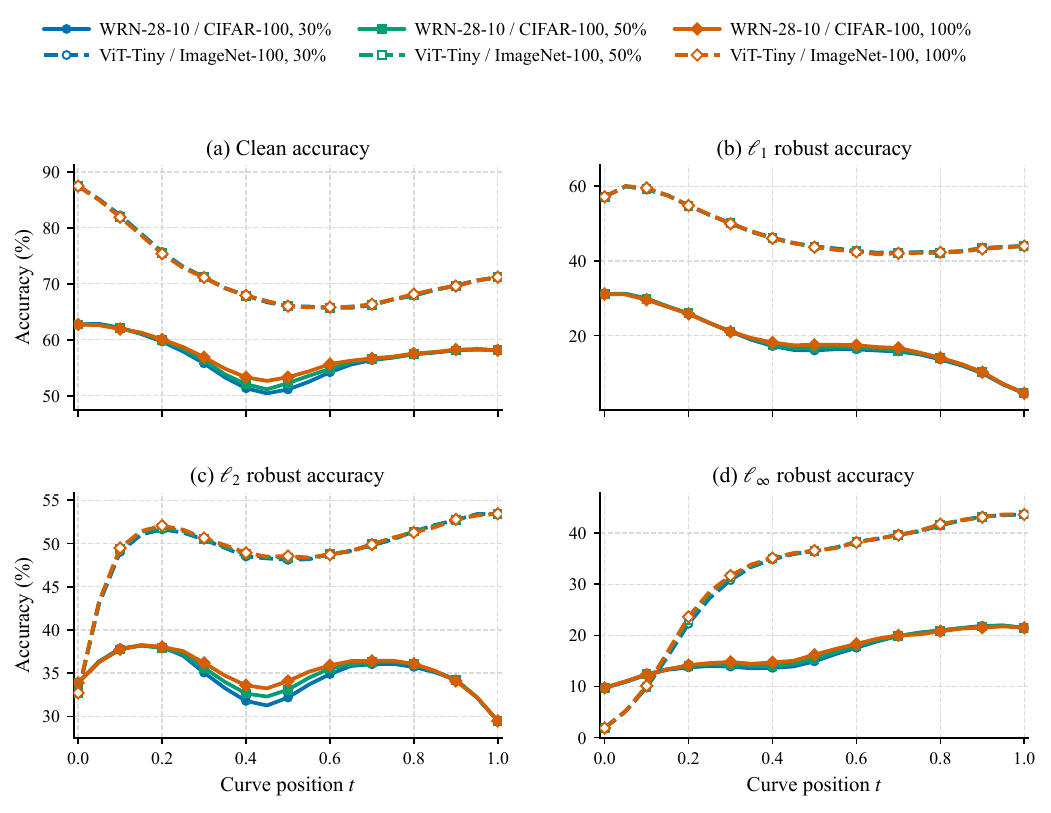}
    \caption{Effect of the partial updating ratio on clean and norm-specific adversarial accuracy along the robust connectivity curve.}
    \label{fig:partial_updating_ratio}
    \vspace{-5pt}
\end{figure}

\paragraph{Effect of the number of MoE layers}
As shown in Table~\ref{tab:moe_layer_number}, increasing the number of MoE layers from $K=3$ to $K=5$ consistently improves both clean and multi-norm adversarial performance on the two datasets. On CIFAR-100, the Union accuracy increases from $18.23\%$ to $19.45\%$, while clean, $\ell_1$, $\ell_2$, and $\ell_\infty$ accuracy also improve. A similar trend is observed on ImageNet-100, where the Union accuracy increases from $39.65\%$ to $42.93\%$. However, further increasing $K$ to $7$ yields only marginal changes and slightly reduces the Union accuracy on both datasets. This indicates that introducing additional high-contribution MoE layers initially provides greater expert-specific capacity, whereas expertizing more layers beyond this point offers diminishing returns. Although increasing $K$ generally improves robustness by introducing more expert-specific capacity, it also increases the training and parameter overhead of the MoE architecture. We therefore use $K=3$ as the default setting to achieve a better trade-off between multi-norm robustness and efficiency.

\begin{table}[htbp]
\centering
\vspace{-5pt}
\caption{Effect of the number of MoE layers on model performance and training cost.}
\label{tab:moe_layer_number}
\resizebox{\columnwidth}{!}{
\begin{tabular}{lccccccc}
\toprule
Dataset & $K$ & Training Time (h) & Clean (\%) & $\ell_1$ (\%) & $\ell_2$ (\%) & $\ell_\infty$ (\%) & Union (\%) \\
\midrule
\multirow{3}{*}{CIFAR-100}
& $3$ & $14.60 \pm 0.90$ & $58.34 \pm 0.64$ & $28.40 \pm 0.70$ & $36.48 \pm 0.75$ & $19.71 \pm 0.27$ & $18.23 \pm 0.67$ \\
& $5$ & $20.97 \pm 0.27$ & $62.10 \pm 0.96$ & $31.98 \pm 0.59$ & $40.31 \pm 0.91$ & $21.81 \pm 0.26$ & $19.45 \pm 0.86$ \\
& $7$ & $35.46 \pm 0.95$ & $62.85 \pm 0.62$ & $31.35 \pm 0.51$ & $40.40 \pm 0.76$ & $21.23 \pm 0.12$ & $19.18 \pm 0.24$ \\
\midrule
\multirow{3}{*}{ImageNet-100}
& $3$ & $15.53 \pm 0.81$ & $73.58 \pm 0.50$ & $67.62 \pm 0.69$ & $54.34 \pm 0.89$ & $41.96 \pm 0.95$ & $39.65 \pm 0.54$ \\
& $5$ & $25.58 \pm 0.12$ & $77.17 \pm 0.49$ & $71.90 \pm 0.48$ & $57.71 \pm 0.90$ & $45.35 \pm 0.11$ & $42.93 \pm 0.24$ \\
& $7$ & $39.13 \pm 0.63$ & $77.40 \pm 0.96$ & $71.13 \pm 0.94$ & $57.95 \pm 0.57$ & $45.59 \pm 0.23$ & $42.35 \pm 0.82$ \\
\bottomrule
\end{tabular}
}
\vspace{-5pt}
\end{table}

\subsection{Effect of Warm-up and Load Balancing}

We adopt the same warm-up strategy as~\cite{zhang2023robust} and the same load-balancing loss as~\cite{riquelme2021scaling} to examine whether these commonly used MoE training components benefit Robust CurveMoE. As shown in Table~\ref{tab:moe_training_components}, neither warm-up nor load balancing provides consistent improvements across the two datasets. On CIFAR-100, warm-up slightly improves Union accuracy, while load balancing yields only marginal gains on several norm-specific metrics. However, these improvements are not consistently observed on ImageNet-100, where the configuration without either component achieves the highest clean, $\ell_1$, and $\ell_\infty$ accuracy. Overall, the performance differences among the three settings are small and do not indicate a systematic benefit from introducing either warm-up or load balancing. We therefore omit both components in Robust CurveMoE, resulting in a simpler training procedure without sacrificing clean or multi-norm robustness.

\begin{table}[htbp]
\centering
\vspace{-5pt}
\caption{Effect of warm-up and load balancing on model performance.}
\label{tab:moe_training_components}
\resizebox{\columnwidth}{!}{
\begin{tabular}{lccccccc}
\toprule
Dataset & Warm-up & Load Balancing & Clean (\%) & $\ell_1$ (\%) & $\ell_2$ (\%) & $\ell_\infty$ (\%) & Union (\%) \\
\midrule
\multirow{3}{*}{CIFAR-100}
& Yes & No  & $58.28 \pm 0.45$ & $28.12 \pm 0.55$ & $36.14 \pm 0.30$ & $19.55 \pm 0.74$ & $18.90 \pm 0.68$ \\
& No  & Yes & $58.18 \pm 0.37$ & $28.62 \pm 0.78$ & $36.81 \pm 0.93$ & $19.77 \pm 0.49$ & $18.43 \pm 0.45$ \\
& No  & No  & $58.34 \pm 0.64$ & $28.40 \pm 0.70$ & $36.48 \pm 0.75$ & $19.71 \pm 0.27$ & $18.23 \pm 0.67$ \\
\midrule
\multirow{3}{*}{ImageNet-100}
& Yes & No  & $73.28 \pm 0.15$ & $67.38 \pm 0.43$ & $54.16 \pm 0.17$ & $41.16 \pm 0.64$ & $39.73 \pm 0.65$ \\
& No  & Yes & $73.30 \pm 0.51$ & $67.51 \pm 0.82$ & $54.79 \pm 0.64$ & $41.37 \pm 0.81$ & $39.53 \pm 0.35$ \\
& No  & No  & $73.58 \pm 0.50$ & $67.62 \pm 0.69$ & $54.34 \pm 0.89$ & $41.96 \pm 0.95$ & $39.65 \pm 0.54$ \\
\bottomrule
\end{tabular}
}
\vspace{-5pt}
\end{table}

\section{Conclusion}

In this work, we presented Robust CurveMoE, an efficient MoE-based framework for multi-norm adversarial defense. Instead of independently training multiple norm-specific experts, Robust CurveMoE derives complementary experts from a robust connectivity curve and exploits their diverse robustness profiles through input-dependent sparse routing. To improve efficiency, we selectively expertize only the layers most sensitive to adversarial objectives and introduce contribution-guided partial curve updating to reduce the cost of robust curve construction. We further restrict expert selection to cross-norm robust candidate models, mitigating vulnerabilities caused by adversarially manipulated routing toward overly specialized experts. Experiments on CIFAR-100 and ImageNet-100 with both convolutional and Transformer architectures demonstrate consistent improvements in clean, norm-specific, and Union accuracy over representative multi-norm and robust mode-connectivity baselines. The ablation studies further show that the proposed partial updating and selective expertization retain strong robustness with reduced training and parameter overhead, while commonly used MoE warm-up and load-balancing strategies provide no consistent benefit. Overall, these results demonstrate that robust mode connectivity provides a structured and efficient source of specialized experts and that selectively integrating them through sparse MoE routing offers an effective approach to multi-norm adversarial robustness.

%\section*{Acknowledgments}
%This should be a simple paragraph before the References to thank those individuals and institutions who have supported your work on this article.

\appendix[Proof of Theorem~\ref{thm:partial_update}]
\label{app:partial_update_proof}

\begin{proof}
Let
\begin{equation}
\widetilde{\boldsymbol{\omega}}_{\mathcal{S}}
=
\boldsymbol{\omega}_0
+
\boldsymbol{P}_{\mathcal{S}}
\boldsymbol{\Delta}^{*}
\end{equation}
be the point obtained by retaining only the optimization displacement in the parameter subspace specified by $\mathcal{S}$. Since
$\boldsymbol{\omega}^{*}
=
\boldsymbol{\omega}_0
+
\boldsymbol{\Delta}^{*}$,
we have
\begin{equation}
\widetilde{\boldsymbol{\omega}}_{\mathcal{S}}
-
\boldsymbol{\omega}^{*}
=
-
\boldsymbol{P}_{\mathcal{S}}^{\perp}
\boldsymbol{\Delta}^{*}.
\end{equation}

By the $L$-smoothness of $\mathcal{J}$,
\begin{equation}
\begin{aligned}
\mathcal{J}
\left(
\widetilde{\boldsymbol{\omega}}_{\mathcal{S}}
\right)
\leq\;&
\mathcal{J}
\left(
\boldsymbol{\omega}^{*}
\right)
+
\left\langle
\nabla
\mathcal{J}
\left(
\boldsymbol{\omega}^{*}
\right),
\widetilde{\boldsymbol{\omega}}_{\mathcal{S}}
-
\boldsymbol{\omega}^{*}
\right\rangle \\
&+
\frac{L}{2}
\left\|
\widetilde{\boldsymbol{\omega}}_{\mathcal{S}}
-
\boldsymbol{\omega}^{*}
\right\|_2^2.
\end{aligned}
\end{equation}
Because $\boldsymbol{\omega}^{*}$ is a stationary point of the full curve optimization problem,
\begin{equation}
\nabla
\mathcal{J}
\left(
\boldsymbol{\omega}^{*}
\right)
=
\boldsymbol{0}.
\end{equation}
Therefore,
\begin{equation}
\mathcal{J}
\left(
\widetilde{\boldsymbol{\omega}}_{\mathcal{S}}
\right)
-
\mathcal{J}
\left(
\boldsymbol{\omega}^{*}
\right)
\leq
\frac{L}{2}
\left\|
\boldsymbol{P}_{\mathcal{S}}^{\perp}
\boldsymbol{\Delta}^{*}
\right\|_2^2.
\end{equation}

It remains to relate
$\widetilde{\boldsymbol{\omega}}_{\mathcal{S}}$
to the optimal solution of the restricted problem. By the definition of
$\mathcal{F}_{\mathcal{S}}$,
choosing
$\boldsymbol{d}=\boldsymbol{\Delta}^{*}$
gives
\begin{equation}
\boldsymbol{\omega}_0
+
\boldsymbol{P}_{\mathcal{S}}
\boldsymbol{d}
=
\widetilde{\boldsymbol{\omega}}_{\mathcal{S}},
\end{equation}
and hence
$\widetilde{\boldsymbol{\omega}}_{\mathcal{S}}
\in
\mathcal{F}_{\mathcal{S}}$.
Since
$\boldsymbol{\omega}_{\mathcal{S}}^{*}$
minimizes $\mathcal{J}$ over
$\mathcal{F}_{\mathcal{S}}$,
\begin{equation}
\mathcal{J}
\left(
\boldsymbol{\omega}_{\mathcal{S}}^{*}
\right)
\leq
\mathcal{J}
\left(
\widetilde{\boldsymbol{\omega}}_{\mathcal{S}}
\right).
\end{equation}
Combining the above inequalities yields
\begin{equation}
\mathcal{J}
\left(
\boldsymbol{\omega}_{\mathcal{S}}^{*}
\right)
-
\mathcal{J}
\left(
\boldsymbol{\omega}^{*}
\right)
\leq
\frac{L}{2}
\left\|
\boldsymbol{P}_{\mathcal{S}}^{\perp}
\boldsymbol{\Delta}^{*}
\right\|_2^2,
\end{equation}
which proves Eq.~\eqref{eq:partial_update_bound}.
\end{proof}

%{\appendices
%\section*{Proof of the First Zonklar Equation}
%Appendix one text goes here.
% You can choose not to have a title for an appendix if you want by leaving the argument blank
%\section*{Proof of the Second Zonklar Equation}
%Appendix two text goes here.}

 % argument is your BibTeX string definitions and bibliography database(s)
%\bibliography{IEEEabrv,../bib/paper}
%

\bibliographystyle{IEEEtran}
\bibliography{ref}

@article{szegedy2013intriguing,
  title={Intriguing properties of neural networks},
  author={Szegedy, Christian and Zaremba, Wojciech and Sutskever, Ilya and Bruna, Joan and Erhan, Dumitru and Goodfellow, Ian and Fergus, Rob},
  journal={arXiv preprint arXiv:1312.6199},
  year={2013}
}

@inproceedings{maini2022perturbation,
  title={Perturbation type categorization for multiple adversarial perturbation robustness},
  author={Maini, Pratyush and Chen, Xinyun and Li, Bo and Song, Dawn},
  booktitle={Uncertainty in Artificial Intelligence},
  pages={1317--1327},
  year={2022},
  organization={PMLR}
}

@inproceedings{touvron2021training,
  title={Training data-efficient image transformers \& distillation through attention},
  author={Touvron, Hugo and Cord, Matthieu and Douze, Matthijs and Massa, Francisco and Sablayrolles, Alexandre and J{\'e}gou, Herv{\'e}},
  booktitle={International conference on machine learning},
  pages={10347--10357},
  year={2021},
  organization={PMLR}
}

@article{zagoruyko2016wide,
  title={Wide residual networks},
  author={Zagoruyko, Sergey and Komodakis, Nikos},
  journal={arXiv preprint arXiv:1605.07146},
  year={2016}
}

@inproceedings{deng2009imagenet,
  title={Imagenet: A large-scale hierarchical image database},
  author={Deng, Jia and Dong, Wei and Socher, Richard and Li, Li-Jia and Li, Kai and Fei-Fei, Li},
  booktitle={2009 IEEE conference on computer vision and pattern recognition},
  pages={248--255},
  year={2009},
  organization={Ieee}
}

@article{krizhevsky2009learning,
  title={Learning multiple layers of features from tiny images},
  author={Krizhevsky, Alex and Hinton, Geoffrey and others},
  year={2009},
  publisher={Toronto, ON, Canada}
}

@article{shazeer2017outrageously,
  title={Outrageously large neural networks: The sparsely-gated mixture-of-experts layer},
  author={Shazeer, Noam and Mirhoseini, Azalia and Maziarz, Krzysztof and Davis, Andy and Le, Quoc and Hinton, Geoffrey and Dean, Jeff},
  journal={arXiv preprint arXiv:1701.06538},
  year={2017}
}

@article{fedus2022switch,
  title={Switch transformers: Scaling to trillion parameter models with simple and efficient sparsity},
  author={Fedus, William and Zoph, Barret and Shazeer, Noam},
  journal={Journal of Machine Learning Research},
  volume={23},
  number={120},
  pages={1--39},
  year={2022}
}

@article{riquelme2021scaling,
  title={Scaling vision with sparse mixture of experts},
  author={Riquelme, Carlos and Puigcerver, Joan and Mustafa, Basil and Neumann, Maxim and Jenatton, Rodolphe and Susano Pinto, Andr{\'e} and Keysers, Daniel and Houlsby, Neil},
  journal={Advances in Neural Information Processing Systems},
  volume={34},
  pages={8583--8595},
  year={2021}
}

@article{puigcerver2022adversarial,
  title={On the adversarial robustness of mixture of experts},
  author={Puigcerver, Joan and Jenatton, Rodolphe and Riquelme, Carlos and Awasthi, Pranjal and Bhojanapalli, Srinadh},
  journal={Advances in neural information processing systems},
  volume={35},
  pages={9660--9671},
  year={2022}
}

@inproceedings{zhang2023robust,
  title={Robust mixture-of-expert training for convolutional neural networks},
  author={Zhang, Yihua and Cai, Ruisi and Chen, Tianlong and Zhang, Guanhua and Zhang, Huan and Chen, Pin-Yu and Chang, Shiyu and Wang, Zhangyang and Liu, Sijia},
  booktitle={Proceedings of the IEEE/CVF International Conference on Computer Vision},
  pages={90--101},
  year={2023}
}

@article{zhang2025optimizing,
  title={Optimizing robustness and accuracy in mixture of experts: A dual-model approach},
  author={Zhang, Xu and Xu, Kaidi and Hu, Ziqing and Wang, Ren},
  journal={arXiv preprint arXiv:2502.06832},
  year={2025}
}

@article{garipov2018loss,
  title={Loss surfaces, mode connectivity, and fast ensembling of dnns},
  author={Garipov, Timur and Izmailov, Pavel and Podoprikhin, Dmitrii and Vetrov, Dmitry P and Wilson, Andrew G},
  journal={Advances in neural information processing systems},
  volume={31},
  year={2018}
}

@inproceedings{draxler2018essentially,
  title={Essentially no barriers in neural network energy landscape},
  author={Draxler, Felix and Veschgini, Kambis and Salmhofer, Manfred and Hamprecht, Fred},
  booktitle={International conference on machine learning},
  pages={1309--1318},
  year={2018},
  organization={PMLR}
}

@inproceedings{frankle2020linear,
  title={Linear mode connectivity and the lottery ticket hypothesis},
  author={Frankle, Jonathan and Dziugaite, Gintare Karolina and Roy, Daniel and Carbin, Michael},
  booktitle={International conference on machine learning},
  pages={3259--3269},
  year={2020},
  organization={PMLR}
}

@inproceedings{wortsman2022model,
  title={Model soups: averaging weights of multiple fine-tuned models improves accuracy without increasing inference time},
  author={Wortsman, Mitchell and Ilharco, Gabriel and Gadre, Samir Ya and Roelofs, Rebecca and Gontijo-Lopes, Raphael and Morcos, Ari S and Namkoong, Hongseok and Farhadi, Ali and Carmon, Yair and Kornblith, Simon and others},
  booktitle={International conference on machine learning},
  pages={23965--23998},
  year={2022},
  organization={PMLR}
}

@article{zhao2020bridging,
  title={Bridging mode connectivity in loss landscapes and adversarial robustness},
  author={Zhao, Pu and Chen, Pin-Yu and Das, Payel and Ramamurthy, Karthikeyan Natesan and Lin, Xue},
  journal={arXiv preprint arXiv:2005.00060},
  year={2020}
}

@article{wang2023robust,
  title={Robust Mode Connectivity-Oriented Adversarial Defense: Enhancing Neural Network Robustness Against Diversified $\backslash$ $ell\_p $ Attacks},
  author={Wang, Ren and Li, Yuxuan and Liu, Sijia},
  journal={arXiv preprint arXiv:2303.10225},
  year={2023}
}

@article{goodfellow2014explaining,
  title={Explaining and harnessing adversarial examples},
  author={Goodfellow, Ian J and Shlens, Jonathon and Szegedy, Christian},
  journal={arXiv preprint arXiv:1412.6572},
  year={2014}
}

@article{madry2017towards,
  title={Towards deep learning models resistant to adversarial attacks},
  author={Madry, Aleksander and Makelov, Aleksandar and Schmidt, Ludwig and Tsipras, Dimitris and Vladu, Adrian},
  journal={arXiv preprint arXiv:1706.06083},
  year={2017}
}

@inproceedings{croce2020reliable,
  title={Reliable evaluation of adversarial robustness with an ensemble of diverse parameter-free attacks},
  author={Croce, Francesco and Hein, Matthias},
  booktitle={International conference on machine learning},
  pages={2206--2216},
  year={2020},
  organization={PMLR}
}

@article{tramer2019adversarial,
  title={Adversarial training and robustness for multiple perturbations},
  author={Tramer, Florian and Boneh, Dan},
  journal={Advances in neural information processing systems},
  volume={32},
  year={2019}
}

@inproceedings{maini2020adversarial,
  title={Adversarial robustness against the union of multiple perturbation models},
  author={Maini, Pratyush and Wong, Eric and Kolter, Zico},
  booktitle={International Conference on Machine Learning},
  pages={6640--6650},
  year={2020},
  organization={PMLR}
}

@inproceedings{croce2022adversarial,
  title={Adversarial Robustness against Multiple and Single $ l\_p $-Threat Models via Quick Fine-Tuning of Robust Classifiers},
  author={Croce, Francesco and Hein, Matthias},
  booktitle={International Conference on Machine Learning},
  pages={4436--4454},
  year={2022},
  organization={PMLR}
}

\newpage

%\section{Biography Section}
%If you have an EPS/PDF photo (graphicx package needed), extra braces are
% needed around the contents of the optional argument to biography to prevent
% the LaTeX parser from getting confused when it sees the complicated
% $\backslash${\tt{includegraphics}} command within an optional argument. (You can create
% your own custom macro containing the $\backslash${\tt{includegraphics}} command to make things
% simpler here.)
% 
%\vspace{11pt}
%
%\bf{If you include a photo:}\vspace{-33pt}
%\begin{IEEEbiography}[{\includegraphics[width=1in,height=1.25in,clip,keepaspectratio]{fig1}}]{Michael Shell}
%Use $\backslash${\tt{begin\{IEEEbiography\}}} and then for the 1st argument use $\backslash${\tt{includegraphics}} to declare and link the author photo.
%Use the author name as the 3rd argument followed by the biography text.
%\end{IEEEbiography}
%
%\vspace{11pt}
%
%\bf{If you will not include a photo:}\vspace{-33pt}
%\begin{IEEEbiographynophoto}{John Doe}
%Use $\backslash${\tt{begin\{IEEEbiographynophoto\}}} and the author name as the argument followed by the biography text.
%\end{IEEEbiographynophoto}

\vfill

\end{document}